\documentclass{article}

\usepackage{microtype}
\usepackage{graphicx}
\usepackage{subfigure}
\usepackage{booktabs}

\usepackage{hyperref}

\usepackage[accepted]{icml2021}

\icmltitlerunning{On the (Un-)Avoidability of Adversarial Examples}

\usepackage{amsmath, amsthm, amssymb,dsfont}
\usepackage{xcolor, tikz}
\usepackage{hyperref}
\usepackage{wrapfig}

\usepackage{stfloats}

\newtheorem{theorem}{Theorem}
\newtheorem{lemma}[theorem]{Lemma}

\newtheorem{definition}{Definition}
\newtheorem{observation}[theorem]{Observation}

   \newcommand{\reals}{\mathbb{R}}
   \newcommand{\naturals}{\mathbb{N}}
   \newcommand{\Ex}{\mathbb{E}}
   \renewcommand{\Pr}{\mathbb{P}}
   \newcommand{\Lo}[1]{{\mathcal L_{#1}}}
   \newcommand{\bLo}[1]{{\mathcal L^{0/1}_{#1}}}
   \newcommand{\rLo}[2]{{\mathcal L^{#1}_{#2}}}
   \newcommand{\arLo}[2]{{\mathcal L^{{#1}-ar}_{#2}}}
   \newcommand{\lo}{\ell}
   \newcommand{\blo}{\ell^{0/1}}
   \newcommand{\rlo}[1]{\ell^{#1}}
   \newcommand{\arlo}{\ell^{ar}}
   \newcommand{\BLo}[1]{{\mathcal L^{B}_{#1}}}
   \newcommand{\rBLo}[1]{{\mathcal L^{rB}_{#1}}}
   \newcommand{\indct}[1]{\mathds{1}\left[{#1}\right]}
   \newcommand{\B}{{\mathcal B}}
   \renewcommand{\P}{{\mathcal P}}
   \newcommand{\A}{{\mathcal A}}

  \newcommand{\F}{{\mathcal F}}
  
  \renewcommand{\H}{{\mathcal H}}

  \newcommand{\X}{\mathcal{X}}
  \newcommand{\Y}{\mathcal{Y}}

  \newcommand{\e}{\mathrm{e}}
  
  \newcommand{\mar}[2]{\mathrm{mar}_{#2}^{#1}}
  
  \newcommand{\err}[1]{\mathrm{err}_{#1}}
  \newcommand{\iid}{i.i.d.~}
   
  \newcommand{\nn}{\mathrm{NN}}
  
  \newcommand{\rbh}[2]{h^{{#1}B}_{#2}}
  \newcommand{\bh}[1]{h^{B}_{#1}}

\begin{document}

\twocolumn[
\icmltitle{On the (Un-)Avoidability of Adversarial Examples}

\begin{icmlauthorlist}
\icmlauthor{Sadia Chowdhury}{york}
\icmlauthor{Ruth Urner}{york}
\end{icmlauthorlist}

\icmlaffiliation{york}{Lassonde School of Engineering, EECS Department, York University, Toronto, Canada}
\icmlcorrespondingauthor{Ruth Urner}{ruth@eecs.yorku.ca}
\icmlkeywords{Machine Learning Theory, Adversarial Loss, Adaptive Robustness}

\vskip 0.3in
]

\printAffiliationsAndNotice{}

\begin{abstract}
The phenomenon of adversarial examples in deep learning models has caused substantial concern over their reliability. While many deep neural networks have shown impressive performance in terms of predictive accuracy, it has been shown that in many instances an imperceptible perturbation can falsely flip the network's prediction. Most research has then focused on developing defenses against adversarial attacks or learning under a worst-case adversarial loss. In this work, we take a step back and aim to provide a framework for determining whether a model's label change under small perturbation is justified (and when it is not). We carefully argue that adversarial robustness should be defined as a locally adaptive measure complying with the underlying distribution. We then suggest a definition for an adaptive robust loss, derive an empirical version of it, and develop a resulting data-augmentation framework. We prove that our adaptive data-augmentation maintains consistency of $1$-nearest neighbor classification under deterministic labels and provide illustrative empirical evaluations.
\end{abstract}

\section{Introduction}
\label{intro}
Deep learning methods have enjoyed phenomenal successes on wide range of applications of predictive tasks  in the past decade.
However, it has been demonstrated that, while these networks are often highly accurate at making predictions on natural data inputs, the performance can degrade drastically when inputs are slightly manipulated \citep{SzegedyZSBEGF13}. Flipping a few pixels in an image, a perturbation that is not perceivable by humans, can lead to misclassification by the trained network. These unexpected, and seemingly erratic behaviors of deep learning models have caused substantial concern over their reliability and trustworthiness. 
Particularly so, if these models are to be employed in applications where vulnerability to manipulations may have fatal consequences (for example if learning based vision technologies are to be employed in self-driving cars).
Recent years have seen a surge in studies aiming to enhance robustness of deep learning \citep{ChakrabortySurvey2018arxiv,GoodfellowMP18,akhtar2018threat}.
Practical approaches are often aimed at smoothing either the trained model or the training data: By data-augmentation the training data gets artificially augmented with perturbations of natural inputs as a way to promote robustness of the model during training \citep{YangZCWCL19,yu2020pda}. Alternatively, a trained model gets smoothed during post-processing, so as to not suffer sudden switches of the output class in areas where natural inputs occur \citep{CohenRK19,SalmanLRZZBY19}. 

Theoretical studies on the problem of adversarial robustness have often focused on exploring how adversarial robustness can be phrased in terms of a modified loss function and how this modified notion of loss affects learnability, both in terms of statistical  and computational aspects \citep{MontasserHS19,YinRB19,GourdeauKK019,montasser2020efficiently,AshtianiPU20}.
However, both theoretical studies and practical heuristics developed in the context of promoting robustness to adversarial attacks, are typically aimed at a fixed notion of smoothness with a fixed degree of perturbations that the model should be made robust to.

In this work, we take a step back, and analyze when a robustness requirement is plausible with respect to the underlying data-generating process. 
It has been observed before that a requirement of hard margins on a learned predictor (enforcing the learned predictor to assign constant output label in balls of fixed size around input points) can be at odds with achieving high accuracy, even if the data-generating distribution, in principle allows for accurate prediction \citep{diochnos2018adversarial,GourdeauKK019}. In this work, we formally argue that robustness requirements should be aligned with the underlying data-generating process, and that such an alignment inherently requires a \emph{locally adaptive notion of robustness, that is, a locally adaptive robust loss}.

More specifically, we introduce a new notion of separability of a distribution, the \emph{margin rate} of the distribution. The margin is a function that measures how much probability mass is assigned to areas that are close to the decision boundary of a (certain type of canonical) Bayes classifier. 
We prove that, given the margin rate of a distribution, a robustness parameter can be chosen so that the optimal predictors have similar loss values (in terms of classification and robust loss). 
However, we also show that choosing the robustness parameter slightly too large, can result in the optimal predictors disagreeing on a proportion of probability mass $1/2$. 
This implies that if the robustness parameter is chosen even slightly too large for the data-generating process at hand, any learning method that is \emph{consistent} (converges to the best possible loss as training data set size increases) \emph{with respect to one loss} is \emph{not consistent with respect to the other}.

This motivates our proposition of redefining the robustness requirement. 
We argue that \emph{robustness is inherently a local property} and that learned predictors should thus satisfy a local notion of robustness that is in line with the underlying data-generating process.
While such a requirement can not readily be phrased as a loss function (that operates on a pair of predictor and input/output data instance), we derive a natural empirical version of this requirement.
This allows for evaluating the requirement on datasets. 
Further, we argue that our notion of locally adaptive robustness yields a natural paradigm for data augmentation, which adheres to the margin properties of the data-generating distribution. We prove that using this form of data-augmentation as a pre-processing step maintains consistency of $1$-nearest neighbor classification on tasks without stochasticity in the labels.

Finally, in Appendix Section \ref{s:experiments} we present a set of \emph{illustrative experiments} for the proposed data-augmentation method and adaptive robust loss in combination with training a ReLU neural network.
The synthetic datasets were designed so as to  highlight the occurrence of adversarial examples when the  data sits on a lower dimensional manifold, a scenario that is considered one of the sources adversarial vulnerability \citep{KhouryHadfield19}. Our experiments visually make the case for the adaptive robust loss in  situations where the label classes have \emph{ different degrees of separation in different parts of the space}.
For lack of space in this extended abstract, we also discuss related work in detail in the Appendix Section \ref{s:relatedwork}.

\section{Formal Setup}\label{s:setup_short}
We provide a full formal setup Section \ref{s:setup} in the Appendix. Here we summarize essential notation.
We let  $\X\subseteq\reals^d$ denote the domain and $\Y=\{0,1\}$ the label space.
We assume that data is generated by some distribution $P$ over $\X\times \Y$.
We say that the distribution has \emph{deterministic labels} if $\Pr_{(x,y)\sim P}[y = 1 \mid x] \in \{0,1\}$ for all $x\in\X$.
A \emph{classifier} or \emph{hypothesis} is a function $h:\X\to\Y$.
We let $\F$ denote the set of all Borel measurable functions $f:\X\to\Y$.
The quality of prediction of a hypothesis on $(x,y)$ is measured by a \emph{loss function} $\lo:(\F\times \X \times \Y) \to \reals$,
for classification problems, typically with the \emph{binary} or \emph{classification loss}: 
\[
\blo(h, x, y) = \indct{h(x) \neq y}.
\]
We denote the \emph{expected loss} (or \emph{true loss}) of a hypothesis $h$ with respect to the distribution $P$ and loss function $\lo$ by $\Lo{P} (h) = \Ex_{(x,y)\sim P} [\lo(h , x, y)]$. 
In particular, we will denote the true binary loss by $\bLo{P}(h)$.
The \emph{empirical loss} of a hypothesis $h$ with respect to loss function $\lo$ and a sample $S = ((x_1, y_1), \ldots, (x_n,y_n))$ is defined as $\Lo{S}(h) = \frac{1}{n}\sum_{i=1}^n \lo(h, x_i, y_i)$.

We consider the most commonly used notion of an (adversarially) robust loss \citep{MontasserHS19,Kamalika2019arxiv}.
For a point $x\in \X$, we let $\B_r(x)$ denote the (open) ball of radius $r$ around $x$.
We then define the robust loss as:
\[
 \rlo{r}(h, x, y) = \indct{\exists z\in \B_{r} ~:~ h(z) \neq y}.
\]
and we let $\rLo{P}{r}(h)$ denote the expected robust loss of $h$.
We have $\rlo{r}(h, x, y) =1$ when $(x,y)$ falls into the \emph{error region}, 
$\err{h} = \{(x,y) \in X\times Y)~\mid~ h(x)\neq y\}$,
or when $x$ lies in the \emph{margin area} $\mar{r}{h}$ of $h$, which we define as:
$\mar{r}{h} =  \{x\in \X ~\mid~ \exists z\in\B_r(x): h(x)\neq h(z)\}.$
The \emph{Bayes classifier} is a classifier that has the minimal true loss with regard to $P$.
We denote the Bayes classifier with respect to the binary loss as $\bh{P}$ and it's loss, the \emph{Bayes risk} by $\BLo{P} = \bLo{P}(h_P^B)$.
We denote the robust-Bayes classifier by $h^{rB}_P$ and the robust-Bayes risk by $\rBLo{P} = \rLo{r}{P}(h^{rB}_P)$.

\section{Relaxations of separability and the margin canonical Bayes}\label{s:relaxedmargins}
 
It has been shown in the literature, that choosing a fixed, unsuitable robustness parameter can lead to inconsistencies between optimaility of binary and robust loss requirements. We review and refine some of these results in the Appendix, Section \ref{ss:binaryversusrobust}.
There we also review that if the distribution is separable (in the sense that $P_\X(\mar{r}{\bh{P}}) = 0$, for some $0/1$-optimal classifier $\bh{P}$), then the robust optimal and $0/1$ optimal predictors coincide. 
However, this is a very strong separability assumption. We start here by relaxing this requirement and showing that, one can choose the robustness parameter $r$ in dependence on ``how separable'' (in a precise sense that we introduce next) the distribution $P$ is and on how close we would like the optimal predictors to be. 

\subsection{Choosing a robustness parameter}\label{ss:choosingr}

Note that, for a fixed predictor $h$, we have $P_\X(\mar{r}{h}) \geq P_\X(\mar{r'}{h})$ if $r \geq r'$. Thus, the function
\[
 \phi_P^h(r) = P_\X(\mar{r}{h})
\]
will monotonically decrease to $0$ as $r$ goes to $0$ for any predictor $h$.
If $h$ is a Bayes predictor, then the rate at which $\phi_P^h(r)$ converges to $0$ as $r\to 0$, can be viewed as a measure of ``how separable'' the data- generating process is, that is, how fast the density of the marginal $P_\X$ vanishes towards the boundary between the two label classes. However, since the Bayes predictor is generally not uniquely defined, we need to specify which Bayes predictor should be employed to measure the separability of the distribution. For simplicity, we will assume here that we have $\mu_P(x) = 0.5$ for the regression function only on a set of measure $0$, and define a margin-canonical Bayes predictor as follows: We let $\X^0\subseteq \X$ denote the closure of the part of the space, where all Bayes classifiers assign label $0$, and let $\X^1\subseteq \X$ the closure of the part of the space where all Bayes classifiers assign label $1$. That is, under the above assumption, the support of the marginal $P_\X$ is $\X^0\cup\X^1$.

We can now define a \emph{margin-canonical Bayes classifier} $\bh{P}$ by nearest neighbor labeling with respect to the sets $\X^0$ and $\X^1$. We only need to specify $\bh{P}(x)$ for points $x$ that are outside the support of $P_\X$. By definition, there exists a ball of some radius $r$ around such a point $x$ that has has no probability mass: $P_\X(\B_r(x)) = 0$. Thus, $x$ has positive distance to both $\X^0$ and $\X^1$ and we will set $\bh{P}(x) = i$ if $\X^i$ is the closer set to $x$, breaking ties arbitrarily. We note that our definitions and results in subsequent sections also hold for the margin rate of any other Bayes classifier. 

\begin{definition}[Margin rate]\label{d:marginrate}
Let $P$ be a distribution over $\X \times \{0,1\}$ and let $\bh{P}$ be the margin-canonical Bayes classifier.
Then we define \emph{margin-rate} of $P$ as the function
$\Phi_P(r) = \phi_P^{\bh{P}}(r).$
If there exists an $r>0$ such that $\Phi_P(r) = 0$, then we call the distribution $P$ \emph{strongly separable}.
\end{definition}

The margin rate is related the notion of \emph{Probabilistic Lipschitzness} \citep{UrnerWB13} and the \emph{geometric noise exponent} \citep{Steinwart_2007}. 
We now show that the margin rate can be used to choose a robustness parameter for which the optimal robust predictor has close to optimal classification loss and vice versa. If the labels of the distribution are deterministic, then we also get closeness as functions of the optimal predictors.

\begin{theorem}\label{thm:chooser}
 Let $P$ be a data-generating distribution over $\X\times \{0,1\}$, let $\Phi_P: \reals^{+} \to [0,1]$ denote its margin rate, and let $\bh{P}$ denote the $0/1$-optimal classifier defining the margin rate. 
 For every $\epsilon >0$, if we let $r\in \Phi_P^{-1}([0,\epsilon])$, then for any $r$-robust optimal classifier $\rbh{r}{P}$ we have
 \[
  \rLo{r}{P}(\bh{P}) \leq \rBLo{P} + \epsilon \quad\text{and}\quad \bLo{P}(\rbh{r}{P}) \leq \BLo{P} + \epsilon.
 \]
 In addition, if the labeling of $P$ is deterministic, we have
 \[
  P_\X[\bh{P} ~\Delta~ \rbh{r}{P}] \leq \epsilon.
 \] 
\end{theorem}

We next argue that, while a separability assumption can yield closeness in loss values of the optimal predictors, it implies closeness of the actual functions only if the labeling is, in addition deterministic.
That is, the assumption of deterministic labels is \emph{necessary} for the second part of the above Theorem (Observation \ref{obs:deterministicnecessary}).

\begin{observation}\label{obs:deterministicnecessary}
Let $\epsilon > 0$ be given. 
There exists a data-generating distribution $P$ over $\reals^2 \times \{0,1\}$ with linear margin rate $\Phi_P: \reals^{+} \to [0,1]$, $\Phi_P(r) = 0.5r$ such that, for any $r\in \Phi_P^{-1}([0,\epsilon])$, we get 
$P_\X[\bh{P} ~\Delta~ \rbh{r}{P}] = \frac{1}{2}$
\end{observation}
Next, we argue that, even under deterministic labels, choosing a robustness parameter slightly larger than implied by Theorem \ref{thm:chooser}, can yield largely differing optimal predictors.  The proof is similar to that of Theorem \ref{thm:differ}.
\begin{observation}\label{obs:rtoolarge}
 Let $\epsilon > 0$ be given. 
 There exists a distribution $P$ over $\reals \times \{0,1\}$ that is strongly separable, such that, for any $r > \sup \Phi_P^{-1}([0,\epsilon])$, we have 
 $P_\X[\bh{P} ~\Delta~ \rbh{r}{P}] = \frac{1}{2}$.
\end{observation}

\subsection{Towards local robustness}\label{ss:towardslocalr}
We now argue that, even if the distribution is strongly separable and the labels are deterministic, then choosing a uniform robustness parameter may not result in the desired outcomes.
To see this, we consider a distribution over domain $\reals^2 \times \{0,1\}$, where the support is distributed uniformly on four points, $(-1, 0.9), (-1, 1.1), (1, 0.9), (1,2)$.
Then predictor $h(x_1, x_2) = \indct{x_2 \geq 1}$ is $0/1$-optimal and also $r$-robust optimal for any $r \leq 0.1$. However, we may prefer a predictor $h^*$ that keeps a larger distance from the point $(1, -.1)$, see illustration in Figure \ref{fig:uniformr} and is equally optimal with respect to the $0.1$-robust loss.

\begin{figure}[h]
 \begin{center}
  \begin{tikzpicture}[scale =1.1]
         
         \foreach \Point in { (-1,1.1), (1, 2)}
         {
          \fill[color=red] \Point circle[radius=1.5pt];   
         }

         \foreach \Point in {(-1,0.9), (1, 0.9)}
         {
           \fill[color=blue] \Point circle[radius=1.5pt];   
         }

         \draw[->] (-3, 0) -- (3,0);          
         \draw[->] (0, -.3) -- (0,2.3);
                 
         \draw[thick, color = orange] (-3, 1) -- (3,1);    
         \draw (3, 1) node[right, color=orange]{$\rbh{0.1}{P}$};
%          \draw (4, .5) node[below right]{$\rbh{0.1}{P}$};
         \draw[dotted, thick, color = green] (-3, 0.5) -- (-1.5,1) -- (-.5,1) -- (3, 2); 
         \draw (3, 2) node[right, color=green]{$h^*$};
  \end{tikzpicture} 
  \end{center}
\caption{Uniform robustness requirement unsuitable.}
\end{figure}
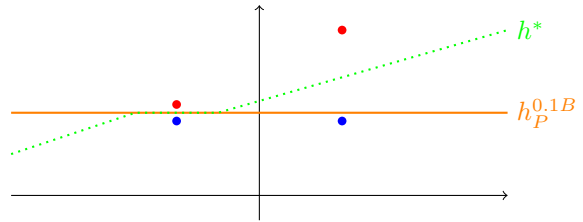\label{fig:uniformr}

\section{Redefining the Robustness Requirement}\label{s:redefiningrobustness}

We have argued (Sections \ref{ss:binaryversusrobust} and \ref{ss:choosingr}) that using a fixed robustness parameter $r$ can lead to inconsistencies (in the sense that the optimal predictors with respect to binary and robust differ vastly) and that even under conditions where the optimal predictors can coincide (strong separability or suitably chosen robustness parameter), optimizing for the robust loss can lead to classifiers that do not reflect our intuition about an optimally robust predictor (Section \ref{ss:towardslocalr}). Ideally we would like a learned predictor to be \emph{everywhere as robust as possible}. We will next formalize this intuition using the notions developed in the previous section. 

\subsection{A local robustness objective}

We propose to phrase robustness in relation to a margin-canonical Bayes predictor. A learned predictor should assign constant label in a ball $\B_r(x)$ around a point $x$ if a margin-canonical Bayes predictor does so. For a predictor $h$ and $x\in\X$, we let $\B^h(x)$ denote the largest ball around $x$ on which $h$ assigns a constant label (possibly $\B^h(x) = \{x\}$).

\begin{definition}[Adaptive robustness]
Let $P$ be a data-generating distribution $\bh{P}$ denote a margin-canonical Bayes predictor, and $h$ an arbitrary predictor.
We define the \emph{adaptive robust loss} $\arlo$ as
\[
 \arlo(h, x, y) = \indct{h(x)\neq y ~\lor~ \B^{\bh{P}}(x) \nsubseteq \B^h(x) }
\]
\end{definition}

This definition implies that, at least for $\bh{P}$ the robust loss coincides with the binary loss.
We note that, similar to the requirement that a predictor should be accurate in a ball of fixed radius, the above proposed loss is not technically a valid loss function, since it depends on $\bh{P}$ rather than just on $h, x$ and $y$. This implies that it can not straightforwardly be estimated from a data-sample. However, we next propose a substitute notion of empirical adaptive robust loss.

\subsection{Empirical adaptive robust loss}
Let $S = ((x_1, y_1), \ldots, (x_n, y_n))$ be a labeled dataset.
For a labeled domain  point $(x,y)$ we let $\rho_S(x)$ denote the distance from $x$ to its nearest neighbor with opposite (or different in the case of more than two classes) label in $S$:
\[
 \rho_S(x,y) = \min_{i\in [n]} \{\|x_i - x\| ~\mid~ (x_i, y_i)\in S, y_i \neq y \}.
\]
In the (degenerate) case that no such point in $S$ has a label different from $y$ (that is, all points in $S$ have the same label), we set $\rho_S(x,y)$ to $\infty$ (or the diameter of the space).
Note that $\rho_S(x,y)$ is well defined for points $(x,y)=(x_i,y_i)\in S$ from the dataset $S$ itself.
We now expand the dataset $S$ by replacing each point with a (constant labeled) ball of radius $c\cdot\rho_S(x_i,y_i)$, for some (to be chosen) constant $c$.

\begin{definition}[$c$-Adaptive robust expansion]
Let $S = ((x_1, y_1), \ldots, (x_n, y_n))$.
% be a labeled dataset.
We call the collection
\[
 S^c = (\B_{c\cdot\rho_S(x_1,y_1)}(x_1, y_1), \ldots, \B_{c\cdot\rho_S(x_n,y_n)}(x_n, y_n))
\]
the \emph{$c$-adaptive robust expansion of $S$}.
\end{definition}

It is easy to see that, as long as $c\leq 1/2$, balls in the $c$-adaptive robust expansion of $S$ overlap only if they have the same label.
Thus, this expansion does not introduce any inconsistencies in the label requirements.
Depending on the geometry of the data-generating process (eg. the curvature of the decision boundary of the regression function) we may also employ larger expansion parameters without introducing inconsistencies.
Using the \emph{$c$-adaptive robust expansion of $S$}, we can define an empirical version of the adaptive robust risk for fixed parameter $c$.
For this, for a predictor $h:\X \to \Y$ and label $y$, we let $h^{-1}(y) \subseteq \X$ denote the part of the domain that $h$ labels with $y$.

\begin{definition}[Empirical $c$-adaptive robust loss]
 Let $c$ be an expansion parameter, $S = ((x_1,y_1),$ \raggedright$ \ldots, (x_n, y_n))$ and $h:\X\to\Y$.
 We define the \emph{empirical $c$-adaptive robust loss of $h$ on $S$} as
 \[
 \arLo{c}{S}(h) = \frac{1}{n}\sum_{i=1}^n \indct{\B_{c\cdot\rho_S(x_i,y_i)}(x_i, y_i) \nsubseteq h^{-1}(y_i)}
 \]
\end{definition}
That is, a point $(x_i, y_i)\in S$ is counted towards the empirical $c$-adaptive robust empirical risk, if $h$ does not label the whole ball $\B_{c\cdot\rho_S(x_i,y_i)}(x_i, y_i)$ in the expanded set with $y_i$.

\subsection{Adaptive robust data-augmentation}

While the empirical $c$-adaptive robust risk is well defined for any predictor $h$ and dataset $S$, it may, computationally, not be straightforward to verify the condition $\indct{\B_{c\cdot\rho_S(x,y)}(x, y) \nsubseteq h^{-1}(y)}$.
A natural estimate is to use $m$ uniform sample points $z^1, \ldots, z^m$ from the ball $\B_{c\cdot\rho_S(x,y)}(x)$ and verify whether $h$ labels all of these with $y$.
Similarly, for training purposes, we may want to use an sample version of the $c$-adaptive robust expansion of $S$. We call this the \emph{$m$-sample-$c$-adaptive robust augmentation of $S$}.
The so augmented dataset $S^{mc}$ is a set of labeled domain points and can be used as a training data-set for a learning algorithm.

\begin{definition}[Adaptive robust data augmentation]
Let  $S = ((x_1, y_1), \ldots,$ \raggedright $(x_n, y_n))$ be a labeled dataset, and $m\in\naturals$.
We call the collection
\[
 S^{mc} = ((z^1_1, y_1), \ldots, (z^m_1, y_1), \ldots (z^1_n, y_n), \ldots, (z^m_n, y_n)),
\]
where every $z^j_i$ is uniformly sampled from the ball $\B_{c\cdot\rho_S(x_i,y_i)}(x_i)$,
the \emph{$m$-sample-$c$-adaptive robust augmentation of $S$}.
\end{definition}

We conjecture that learners, that are consistent with respect to binary loss, remain consistent when fed a $c$-adaptive robust augmentation of $S$ for $c \leq 1/2$. %\ruth{maybe too strong a statement?}
We prove this for a $1$-nearest neighbor classification under deterministic labels. 
This result serves as evidence that our adaptive data augmentation does not induce any inconsistencies with the accuracy requirements.
It holds for a $c$-robust augmentation and any $m$-sample-$c$-robust augmentation if $c\leq 0.5$.
\begin{theorem}\label{thm:NNwithaugmentation}
Let $P$ be a distribution over $[0,1]^d\times\{0,1\}$ with deterministic labels and margin rate $\Phi_P(r)$. Let $\epsilon, \delta >0$ be given.
 Then, with probability at least $1-\delta$ over an is an \iid sample $S$ of size 
 $n \geq \frac{3^dd^{0.5d}}{\e \Phi_P^{-1}(\epsilon)^d \epsilon \delta }$
from $P$, the a $1$-nearest neighbor predictor $h^{0.5}_{\nn}$ on a $m$-sample-$0.5$-adaptive robust augmentation of $S$ satisfies $\bLo{P}(h^{0.5}_{\nn}) $ \raggedright $\leq \epsilon$ for any $m \geq 1$.
\end{theorem}

\section*{Acknowledgements}
This work was supported by an NSERC discovery grant.

\bibliography{refs}
\bibliographystyle{icml2021}

\newpage
\appendix
\section{Related Work}\label{s:relatedwork}

Enhancing robustness to adversarial attacks has received an enormous amount of research attention in recent years, in particular in terms of practical advancements \citep{ChakrabortySurvey2018arxiv, GoodfellowMP18, akhtar2018threat, OnEvaluatingAdvRobArxiv, IlyasSTETM19}. We will focus our discussion of prior work on studies relating to theoretical aspects of learning under robust loss.

Numerous recent theoretical studies focus on the parametric setup and analyze how introducing a robustness requirement may affect statistical convergence of the induced loss classes \citep{cullina2018pac, schmidt2018adversarially, MontasserHS19, YinRB19, AshtianiPU20}, whereas others have focused on computational implications \citep{awasthi2019robustness, montasser2020efficiently}. In particular, that there can be arbitrarily large gaps between the sample complexity of learning a hypothesis with respect to  classification versus roust loss \citep{cullina2018pac, MontasserHS19}. Several studies have derived convergence bounds for classification under adversarial manipulations for fixed hypothesis classes \citep{feige2015learning, attias2018improved, BubeckLPR19}.

Most related to our work are recent studies that also discuss possible options (and their implications) for phrasing a robust loss \citep{diochnos2018adversarial,GourdeauKK019}, and in particular studies that pointed out and analyzes the trade-off between accuracy and robustness \citep{gal2018sufficient, TsiprasSETM19, YangRZSC20}. In particular, a recent study systematically explored the relationship between (a notion of local) Lipschitzness of a nearest neighbor predictor and its robustness. Further closely related to our work are recent studies that analyze and derive properties of optimal predictors under the robust loss and their relation to nearest neighbor predictors \citep{WangJC18, BhattacharjeeC20, YangRWC20}. The latter work studies non-parametric learning for robust classification and proposes a method of data-preprocessing, and, similar to our result for $1$-Nearest Neighbor prediction, proves implied consistency. However, the pre-processing in that study consists of pruning rather than augmenting the data. However, robustness in these prior works is considered with respect to a fixed robustness parameter. In this work, we carefully argue that adversarial robustness should instead be phrased as a locally adaptive requirement. Very recently, a similar argument has independently been made \citep{BCadaptiverobustness21}.

Finally, we note that relationship between non-parametric methods and local adaptivity is well established and our work builds on this. In particular, it has been shown shown that nearest neighbor methods' convergence can be understood and quantified in terms of local smoothness properties of the underlying data-generating process for regression \citep{Kpotufe11} as well as for classification tasks \citep{ChaudhuriD14}.

\section{Formal Setup}\label{s:setup}
\subsection{Basic notions of statistical learning}
We employ a standard setup of statistical learning theory for classification.
We let  $\X\subseteq\reals^d$ denote the domain and $\Y$ (mostly $\Y=\{0,1\}$) a (binary) label space.
We assume that data is generated by some distribution $P$ over $\X\times \Y$ and let $P_\X$ denote the marginal of $P$ over $\X$.
Further, we use notation $\mu_P(x) = \Pr_{(x,y)\sim P}[y = 1 \mid x]$ to denote the \emph{regression function} of $P$.
We say that the distribution has \emph{deterministic labels} if $\mu_P(x) \in \{0,1\}$ for all $x\in\X$.
A \emph{classifier} or \emph{hypothesis} is a function $h:\X\to\Y$.
We let $\F$ denote the set of all Borel measurable functions
from $\X$ to $\Y$ (or all functions in case of a countable domain). 
A \emph{hypothesis class} is a subset of $\F$, often denoted by $\H\subseteq \F$.

The quality of prediction of a hypothesis on an input/output pair $(x,y)$ is measured by a \emph{loss function} $\lo:(\F\times \X \times \Y) \to \reals$.
For classification problems, the quality of prediction is typically measured with the \emph{binary} or \emph{classification loss}: 
\[
\blo(h, x, y) = \indct{h(x) \neq y}, 
\]
where $\indct{\alpha}$ denotes the indicator function for predicate $\alpha$.

We denote the \emph{expected loss} (or \emph{true loss}) of a hypothesis $h$ with respect to the distribution $P$ and loss function $\lo$ by $\Lo{P} (h) = \Ex_{(x,y)\sim P} [\lo(h , x, y)]$. 
In particular, we will denote the true binary loss by $\bLo{P}(h)$.
The \emph{Bayes classifier} is a (in general not unique) classifier  which has the minimal true loss with regard to $P$.
We denote the Bayes classifier with respect to the binary loss as $\bh{P}$ and it's loss, the \emph{Bayes risk} by $\BLo{P} = \bLo{P}(h_P^B)$

The \emph{empirical loss} of a hypothesis $h$ with respect to loss function $\lo$ and a sample $S = ((x_1, y_1), \ldots, (x_n,$ \newline$y_n))$ is defined as $\Lo{S}(h) = \frac{1}{n}\sum_{i=1}^n \lo(h, x_i, y_i)$.

A \emph{learner} $\A$ is a function that takes in a finite sequence of labeled instances $S = ((x_1, y_1), \ldots, (x_n, y_n))$ and outputs a hypothesis $h = \A(S)$.
The following notion of a \emph{consistent learner} captures a basic desirable property: as the learner sees larger and larger samples from the data-generating distribution, the loss of the learner's output should converge to the Bayes risk.
\begin{definition}[Consistency]
 We say that a learner $\A$ is \emph{consistent} with respect to a set of distributions $\P$ if, for every $P\in\P$, every $\epsilon,\delta >0$ we have there is a sample-size $n(P, \epsilon, \delta)$ such that, for all $n \geq n(P, \epsilon, \delta)$, we have
 \[
  \Pr_{S\sim P^n}\left[ \Lo{P}(\A(S)) \leq \BLo{P} + \epsilon \right] \geq 1-\delta
 \]
 We say that $\A$ is \emph{universally consistent}, if $\A$ is consistent with respect to the class of all data-generating distributions.
\end{definition}

\subsection{(Adversarially) robust loss}
We consider the most commonly used notion of an (adversarially) robust loss \citep{MontasserHS19,Kamalika2019arxiv}.
For a point $x\in \X$, we let $\B_r(x)$ denote the (open) ball of radius $r$ around $x$.
We then define the robust loss as:
\[
 \rlo{r}(h, x, y) = \indct{\exists z\in \B_{r} ~:~ h(z) \neq y}.
\]
and we let $\rLo{P}{r}(h)$ denote the expected robust loss of $h$.

As has been done in the literature, we decompose the robust loss into error and margin areas \citep{ZhangYJXGJ19, AshtianiPU20}: We have $\rlo{r}(h, x, y) = 1$ if and only if $h$ makes a mistake on $x$ with respect to label $y$, or, there is an $r$-close instance $z\in \B_r(x)$ that $h$ labels different than $x$, that is, $x$ is $r$-close to $h$'s decision boundary. 

The first condition holds when $(x,y)$ falls into the \emph{error region}, 
$\err{h} = \{(x,y) \in X\times Y)~\mid~ h(x)\neq y\}.$
The second condition holds when $x$ lies in the \emph{margin area} of $h$. 
We define the \emph{margin area} of $h$, as the subset $\mar{r}{h}\subseteq X$  defined by
\[
\mar{r}{h} =  \{x\in \X ~\mid~ \exists z\in\B_r(x): h(x)\neq h(z)\}
\]

We can define notions of a Bayes classifier, and consistency of a learner $\A$ with respect to the robust loss analogously to these notions for the binary loss.
We will denote the robust-Bayes classifier by $h^{rB}_P$ and the robust-Bayes risk by $\rBLo{P} = \rLo{r}{P}(h^{rB}_P)$.
We will often simply refer to the Bayes predictors as the $0/1$-optimal or the $r$-robust optimal predictors.
We note that these optimal predictors are not unique, in particular in the case that the support of the marginal $P_\X$ does not cover the full space. For example, if the data-generating distribution is supported on a lower dimensional manifold, then a $0/1$-optimal predictor is only uniquely determined on that manifold (and even there only with exception of $0$-mass subsets and not in areas with $\mu_P(x) = 0.5)$. Similarly, $r$-robust optimality can be fulfilled by various predictors if the data-generating distribution is strongly separable (see Definition \ref{d:marginrate}).
Explicit forms (analogous to the $0/1$-Bayes being a threshold of the regression function) of the $r$-robust optimal predictor have been derived in the literature (\cite{YangRWC20}).

\section{Robustness and Margins}\label{s:robustnessandmargins}

In this section, we investigate implications of the existence of a low robust-loss classifier and differences between low binary and low robust loss.
We show that the optimal classifiers with respect to these losses can differ significantly, implying that optimizing for one can strongly hurt performance with respect to the other. We then analyze the relationship between the existence of robust classifiers and margin (or separability) properties of the underlying data-generating process.  We argue that, while separability implies the existence of robust classifiers with respect to \emph{some} robustness parameter $r$, using a fixed robustness parameter can again contravene the intention of deriving predictors that are both accurate and as robust as possible.

\subsection{Binary optimal versus robust optimal}\label{ss:binaryversusrobust}
It has been shown before that the definition of the $r$-robust loss implies that, even in situations where the $0/1$-Bayes risk is $0$, that is where the labels are deterministic, no classifier may have $0$ robust loss \citep{diochnos2018adversarial, TsiprasSETM19, ZhangYJXGJ19, GourdeauKK019}:
The existence of a classifier $h$ with $\rLo{r}{P}(h) = 0$ implies that the distribution is \emph{separable}, that is, $P_\X$ is supported on $r$-separated regions of $\X$ and these regions are label-homogeneous. Namely, $\rLo{r}{P}(h) = 0$ implies $\bLo{P}(h) = 0$, which means that the labeling of $P$ is deterministic. In addition, we must have $P(\mar{r}{h}) = 0$, which implies that any point $x$ in the support of $P_\X$ with $h(x) = 1$ has distance at least $2r$ from any point in that support with $h(x) = 0$. In this case, this function $h = \bh{P} = \rbh{r}{P}$ is optimal with respect to both losses. 

In this subsection we inspect the potential tension between robustness and accuracy with an emphasis on the role that stochasticity of the labels play in this phenomenon.
We start by observing that even if the labels are not necessarily deterministic, the optimal robust loss is strictly larger than the optimal $0/1$-loss if and only if a Bayes classifier does not have a strict margin.
\begin{theorem}\label{thm:idenitcalIFmargin}
 We have $\rBLo{P} = \BLo{P}$ if and only if there exists a $0/1$-optimal classifier $\bh{P}$ with $$P_\X(\mar{r}{\bh{P}}) = 0.$$
\end{theorem}
\begin{proof}
 We first assume that $P_\X(\mar{r}{h}) > 0$ for all classifiers $h$ that are $0/1$-optimal. We fix one of them and denote it by $\bh{P}$. Then $\rLo{r}{P}(\bh{P}) > \Lo{P}(\bh{P}) = \BLo{P}$, since on every point in its margin area, $\bh{P}$ suffers binary loss at most $0.5$, while it suffers robust loss $1$. Outside the margin area the loss contributions are identical for both loss functions. Furthermore, for any classifier $h$ that is not $0/1$-optimal, we have $\rLo{r}{P}(h) \geq \bLo{P}(h) > \BLo{P}$. Thus, independently of whether an optimal robust classifier $\rbh{r}{P}$ is also $0/1$-optimal or not, we have
 \[
  \rBLo{P} = \rLo{r}{P}(\rbh{r}{P}) > \BLo{P}
 \]
 As for the other direction, if there is a $0/1$-optimal classifier $\bh{P}$ with $P_\X(\mar{r}{\bh{P}}) = 0$, then it follows immediately, that this classifier is also optimal with respect to the robust loss and its robust loss is identical to its binary loss. Thus $\rBLo{P} = \BLo{P}$.
\end{proof}

Moreover, we will now see, that if the data-generating distribution does not have a margin in the above strong sense, 
then the optimal classifiers with respect to $0/1$-loss and $r$-robust loss can differ significantly \emph{as functions}. The construction for the below result has (in very similar form) appeared in earlier work \citep{ZhangYJXGJ19}. 

\begin{theorem}\label{thm:differ}
Let $r>0$ be a robustness parameter. There exist distributions $P$ such that any predictors $\bh{P}$ and $\rbh{r}{P}$ that are optimal with respect to $0/1$-loss and $r$-robust loss respectively, satisfy
$
  P_\X[\bh{P} ~\Delta~ \rbh{r}{P}] ~=~ \frac{1}{2} ,
$
where 
$
\bh{P} ~\Delta~ \rbh{r}{P} ~=~ \{x \in \X ~\mid \bh{P}(x) ~\neq~ \rbh{r}{P}(x)\}
$
is the set of domain points on which the two optimal classifiers differ.
\end{theorem}
\begin{proof}
 We consider a distribution $P$, where $P_\X$ is supported (uniformly) on just two points $x_0$ and $x_1$ at distance less than $r$ from each other.
 $x_0$ is always generated with label $0$ and $x_1$ is always generated with label $1$. Clearly, the $0/1$-optimal classifier $\bh{P}$ labels accordingly: $\bh{P}(x_0) = 0$ and $\bh{P}(x_1) = 1$, resulting in $\bLo{P}(\bh{P}) = 0$. However, this classifier has largest possible $r$-robust loss: $\rLo{r}{P}(\bh{P}) = 1$, since both points are at distance less than $r$ from a point that $\bh{P}$ labels differently. On the other hand, any constant function $h_c$ has robust loss $\rLo{r}{P}(h_c) = 1/2$, since it's margin has weight $0$ and it mislabels with probability $1/2$. This is optimal with respect to the $r$-robust loss. Thus, we showed that
 $
  P_\X[\bh{P} ~\Delta~ \rbh{r}{P}] ~=~ \frac{1}{2}.
$
\end{proof}
This example shows that binary and robust optimal predictors can differ in half the area of the space. In particular, when the robustness parameter is not chosen suitably, optimizing for one can be strongly sub-optimal (incurring regret of $1/2$) for the other.
This means that \emph{any learning method, will be inconsistent with respect to one of the two losses in question}.

Of course, in the above example, the robustness parameter and distribution are constructed to not match suitably. 

\section{Proofs}
In this section, we list the proofs that were omitted from the main part for lack of space.

\subsection{Proofs from Section \ref{ss:choosingr}}

\begin{proof}[Proof of Theorem \ref{thm:chooser}]
 Due to the way we chose the robustness parameter $r$ here, we immediately get
 \[
  \rLo{r}{P}(\bh{P}) \leq \bLo{P}(\bh{P}) + \epsilon =  \BLo{P} + \epsilon
 \]
 since $P(\mar{r}{\bh{P}}) \leq \epsilon$.
 We need to argue, that no other classifier $h$ can have significantly smaller robust loss.
 As in the proof of Theorem \ref{thm:idenitcalIFmargin}, we observe that, we have $\rLo{r}{P}(h) \geq \bLo{P}(h) \geq \BLo{P}$ for any classifier $h$.
 Thus, in particular $\rLo{r}{P}(\rbh{r}{P}) = \rBLo{P} \geq \BLo{P}$, which yields the first claim.
 
 For the second inequality observe that $\bh{P}$ has $r$-robust loss at most $\BLo{P} + \epsilon$ by choice of $r$. Any robust-optimal classifier $\rbh{r}{P}$ therefore has robust loss at most $\BLo{P} + \epsilon$, which implies that its binary loss is bounded by the same quantity.
 
 Now we assume that the labeling of $P$ is deterministic. 
 This implies that  $\bLo{P}(\bh{P}) = 0$, thus $\rLo{r}{P}(\bh{P}) = \P_\X(\mar{r}{\bh{P}})$.
 Let $\rbh{r}{P}$ be a robust-optimal classifier. 
 By definition of being robust-optimal, we have  $\rLo{r}{P}(\rbh{r}{P}) \leq \rLo{r}{P}(\bh{P}) = \P_\X(\mar{r}{\bh{P}}) \leq \epsilon$.
 Thus, in particular $\bLo{P}(\rbh{r}{P}) \leq \epsilon$, which, in the case of deterministic labels implies 
 $P_\X[\bh{P} ~\Delta~ \rbh{r}{P}] \leq \epsilon$.
\end{proof}

\begin{proof}[Proof of Observation \ref{obs:deterministicnecessary}]
 We consider with uniform marginal over two rectangles in $\reals^2$: We set $R_1 = [-2,-1] \times [-1, 1]$ and $R_2 = [1,2] \times [-1,1]$.
 Further, we set the regression function 
 \[
\mu(x_1, x_2) = \left\{\begin{array}{l}
                          \frac{1}{2} + \frac{\epsilon}{2} \text{ if } x_2 \geq 0\\
                          \frac{1}{2} - \frac{\epsilon}{2} \text{ if } x_2 \leq 0\\
                         \end{array}\right.
\]
 Now it follows that a $0/1$-optima predictor is
 $\bh{P} = \indct{x_2 \geq 0}$ while, for any $r\geq\epsilon/2$, we have 
 $\rbh{r}{P} = \indct{x_1 \geq 0}$, thus $P_\X[\bh{P} ~\Delta~ \rbh{r}{P}] = \frac{1}{2}$.
\end{proof}

\subsection{Proof of Theorem \ref{thm:NNwithaugmentation}}
We will employ a similar proof technique as in Chapter 19 of \citep{shalev2014understanding}. In particular, we will employ Lemma 19.2 therein:

\begin{lemma}[Lemma 19.2 in \citep{shalev2014understanding}]\label{lem:cells}
Let $C_1, C_2, \ldots C_t$ be a collection of subsets of some domain set $\X$. Let $D$ be a distribution over $\X$ and $S$ be an iid sample from $P$ of size $n$. Then
\[
 \Ex_{S\sim D^n} \left[ \sum_{i: C_i \cap S = \emptyset} D(C_i) \right] ~\leq~ \frac{t}{n\cdot\e}
\]

\end{lemma}

Recall that, for a labeled sample $S$, the collection 
\[
 S^c = (\B_{c\cdot\rho_S(x_1,y_1)}(x_1, y_1), \ldots, \B_{c\cdot\rho_S(x_n,y_n)}(x_n, y_n))
\]
denotes the \emph{$c$-adaptive robust expansion of $S$}. We will prove the theorem using this expansion for $c = 0.5$, but note, that the proof (and thus the Theorem) holds equally for
\[
 S^{mc} = ((z^1_1, y_1), \ldots, (z^m_1, y_1), \ldots (z^1_n, y_n), \ldots, (z^m_n, y_n)),
\]
the \emph{$m$-sample-$c$-adaptive robust augmentation of $S$}
(where every $z^j_i$ is uniformly sampled from the ball $\B_{c\cdot\rho_S(x_i,y_i)}(x_i)$), for any $m$.

\begin{proof}[Proof of Theorem \ref{thm:NNwithaugmentation}] 
Let $P$ be a distribution over $[0,1]^d\times\{0,1\}$ with deterministic labels and margin rate $\Phi_P(\cdot)$. 
We let $\bh{P}$ be a margin optimal Bayes predictor for $P$. Note that, since the labels of $P$ are deterministic $\bLo{P}(\bh{P}) = 0$.
Further, we let $\epsilon$ and $\delta$ be given and set $r = \Phi_P^{-1}(\epsilon)$ (to mean the largest $r$, such that $\Phi_P(r) \leq \epsilon$). Further, we set $r' = r/3$.

We can now partition the space $[0,1]^d$ into $t= \frac{\sqrt{d}^d}{r'}$ many sub-cubes of side-length $r'/\sqrt{d}$ and thus diameter $r'$. We denote the cells in this partition by $C_1, \ldots, C_t$.
 
We now let $S$ be a labeled sample and let $h_S^c = h_S^{.5}$ be the nearest neighbor classifier on the $.5$-adaptive robust expansion of $S$. We now bound the mass of points $x$ on which $h_S^c$ makes a false classification by noting that $h_S^c(x) \neq \bh{P}(x)$ implies that one of these two conditions hold:\\
\begin{description}
 \item[C1:] $x$ falls into a cell $C_k$ that has empty intersection with the sample $S$
 \item[C2:] there is at least one sample point $(x_i, y_i)\in S$ in the same cell $C_k$ as $x$, and either there exists such an $(x_i, y_i)\in S$ with $y_i \neq \bh{P}(x)$; or we have  $y_i = \bh{P}(x)$ for all $(x_i, y_i)$ in the same cell, but there is another sample point $(x_j, y_j)\in S$ (in a different cell) with $y_j \neq \bh{P}(x)$ and $x$ is closer to the expansion $\B_{c\cdot\rho_S(x_j,y_j)}(x_j, y_j)$ of $x_j$ than to the expansion $\B_{c\cdot\rho_S(x_i,y_i)}(x_i, y_i)$ of $x_i$
\end{description}

If $S$ is an iid sample from $P$, then, by Lemma \ref{lem:cells} the expected mass of points $x$ cells that are not hit by the sample $S$ is bounded by $\frac{t}{n\cdot\e} = \frac{3^d\sqrt{d}^d}{\Phi_P^{-1}(\epsilon)^d\cdot n \cdot \e}$.
By Markov's inequality, this implies
\[
 \Pr_{S\sim P^n} \left[ \sum_{i: C_i \cap S = \emptyset} P_\X(C_i) > \epsilon \right] ~\leq~ \frac{3^d\sqrt{d}^d}{\epsilon \Phi_P^{-1}(\epsilon)^d\cdot n \cdot \e}
\]
Setting this to $\delta$ shows that, with probability at least $1-\delta$ over a sample $S$ of size 
\[
 n \geq \frac{3^d\sqrt{d}^d}{\epsilon \Phi_P^{-1}(\epsilon)^d\cdot \delta \cdot \e}
\]
the mass of points that fall into ``error case'' C1 is bounded by $\epsilon$.
We now argue that the mass of points that fall into ``error case'' C2 is also bounded by $\epsilon$ by showing that such points actually fall into the $r$-margin area of $\bh{P}$ and, by choice of $r$ and by definition of $\Phi_P$, we have $P_\X(\mar{\bh{P}}{r}) \leq \epsilon$.

Consider a point $x$ in case C2. If there exist a point $(x_i, y_i)\in S$ in the same cell as $x$ with $y_i \neq \bh{P}(x)$, then by the choice of the size of the cells $x\in \mar{\bh{P}}{r'} \subseteq \mar{\bh{P}}{r}$. 

Now consider the other sub-case of C2: There exists at least one point $(x_i, y_i)\in S$ in the same cell as $x$ and all points in the same cell as $x$ have label $\bh{P}(x)$. 
But there is another sample point $(x_j, y_j)\in S$ (in a different cell) with $y_j \neq \bh{P}(x)$ and $x$ is closer to the expansion $\B_{c\cdot\rho_S(x_j,y_j)}(x_j, y_j)$ of $x_j$ than to the expansion $\B_{c\cdot\rho_S(x_i,y_i)}(x_i, y_i)$ of $x_i$, where $c= 0.5$.

Recall that $\rho_S(x_j,y_j)$ is the distance between $x_j$ and a point in $S$ of opposite label to $y_j$. 
We now set $\rho = 0.5 \cdot \rho_S(x_j,y_j)$ for short, that is $\rho$ is the radius of the expansion of $(x_j, y_j)$.

Since the cell that $x$ is in also contains $(x_i, y_i)$ and $y_i \neq y_j$ in this sub-case, we know that $2\rho \leq \|x_i-x_j\|$. 
Further, we know $\|x_i-x\| \leq r' = r/3$ since $x_i$ in in the same cell as $x$.

Let $z\in\B_{c\cdot\rho_S(x_j,y_j)}(x_j, y_j)$ be the point in $\B_{c\cdot\rho_S(x_j,y_j)}(x_j, y_j)$ closest to $x$. 
Then, since $x$ is closer to the expansion of $x_j$ than the expansion of $x_i$, we can infer $\|x-z\|\leq r' = r/3$.
This implies $\|z - x_i\|\leq 2r'$.

Now, by the triangle inequality then implies 
\[
 \|x_i - x_j\| \leq \|x_i - z\| + \|z - x_j\| = \|x_i - z\| + \rho,
\]
 thus
 \[
2\rho \leq   \|x_i - x_j\| \leq \|x_i - z\| + \rho 
 \]
which implies
\[
 \rho \leq \|x_i - z\| \leq 2r'.
\]
Now, again invoking the triangle inequality, we can bound the distance between $x$ and $x_j$:
\[
 \|x - x_j\| \leq \|x - z\| + \|z - x_j\| = r' + 2r' = r
\]
Thus, in this case, $x$ also falls into the $r$-margin area of $\bh{P}$ since $\bh{P}(x) \neq \bh{P}(x_j)$.
\end{proof}

\section{Visualizations}\label{s:experiments}
To further validate our proposed adaptive robust data augmentation method, we present a set of illustrative experiments on various synthetic datasets.
To allow for visualizations, we generate data from a ``lower-dimensional manifold'' in two dimensions. 
It has been conjectured that the data being supported on a lower-dimensional manifold is a source of the phenomenon of vulnerability to small perturbations \citep{KhouryHadfield19}. 
Our visualizations in in Figure \ref{fig:allvisualizations} illustrate this phenomenon.

The original support (the data-manifold) of data generating distributions can be seen as the green and blue lines in the first column of Figure \ref{fig:allvisualizations}, blue and green points representing points from the two classes. 
We term our synthetic shapes in Figure \ref{fig:allvisualizations} {\bf Sines, S-figure, NNN, circles, boxes}.
We train a ReLU Neural Network with $2$-hidden layers (of 10 neurons each) data points drawn from these shapes. 
The labeling behavior of the trained network is visualized over the ambient space in red and purple.
The first image in each row depicts the original, labeled data together with the network trained on the original data.

We see in those left-most illustration that without any augmentation, the network's decision boundary is often located close to the data-manifold.
Since the data is supported only on the lower-dimensional manifold, there is no incentive for the decision boundary to keep a distance from the data-manifold.
While the network labels areas on the manifold itself correctly, this behavior leads to the existence of points that are vulnerable to adversarial perturbations: a small deviation away from the data-manifold can lead to a different labeling by the network.

We then augment the training datasets with both fixed and adaptive expansion parameter and train ReLU Neural Networks of the same size on the augmented datasets. The remaining images in each row again illustrate the augmented datasets (green and blue) together with the labeling behaviors of the resulting networks. The last image in each row corresponds to the adaptive augmented data, while the intermediate images correspond to augmentations with increasing, but fixed expansion parameters.

For fixed expansion parameter, we iteratively increase the parameter in a fix sequence, $(0.1, 0.5, 1, 2,....,16)$. 
These expansion parameters were chosen based on the range of the attribute values in the datasets. 
For each sample in a $d$-dimensional dataset, a $d$-dimensional sphere is generated where the radius is the fixed-parameter and the current sample is the center of the sphere. Four new points are then generated in this sphere for each sample. Hence, the training dataset is expanded to four times its original size after fixed-parameter expansion.

\begin{figure*}[b]
\begin{center}

 \includegraphics[width=.21\textwidth]{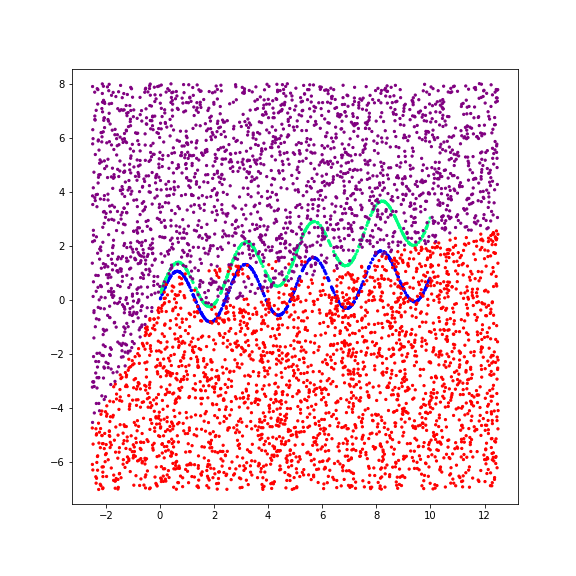}
 \hspace{-.50cm}
 \includegraphics[width=.21\textwidth]{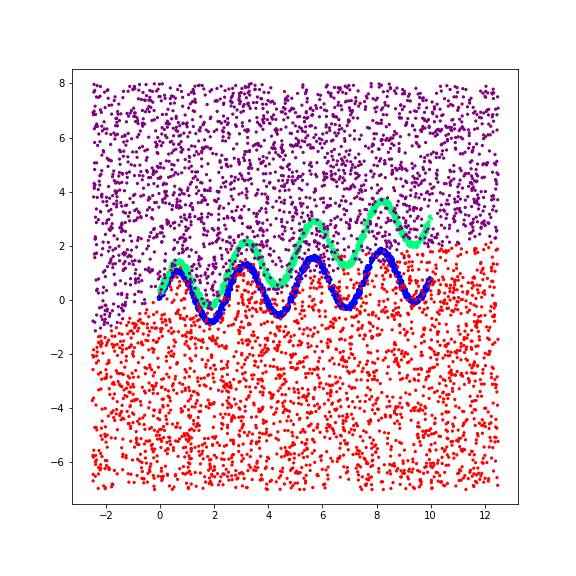}
 \hspace{-.50cm}
 \includegraphics[width=.21\textwidth]{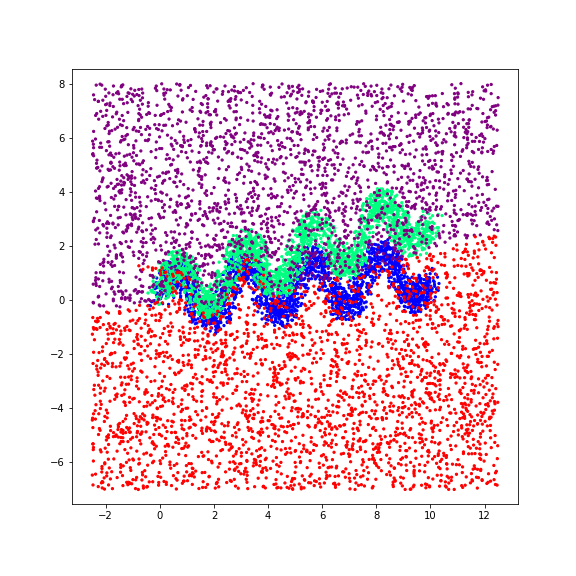}
 \hspace{-.50cm}
 \includegraphics[width=.21\textwidth]{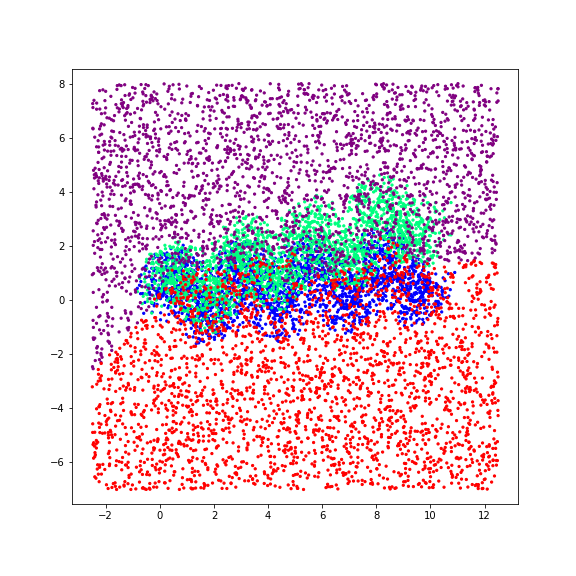}
 \hspace{-.50cm}
 \includegraphics[width=.21\textwidth]{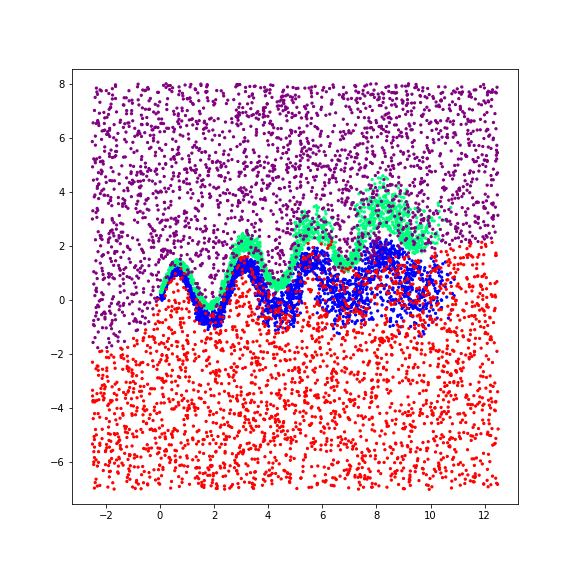}

 \includegraphics[width=.18\textwidth]{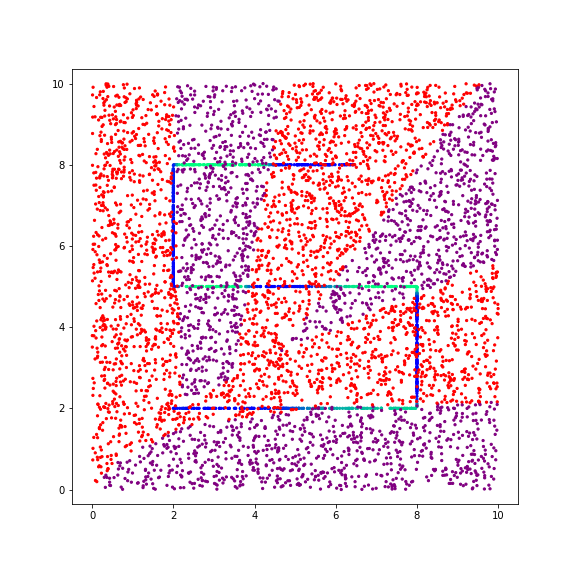}
 \hspace{-.54cm}
 \includegraphics[width=.18\textwidth]{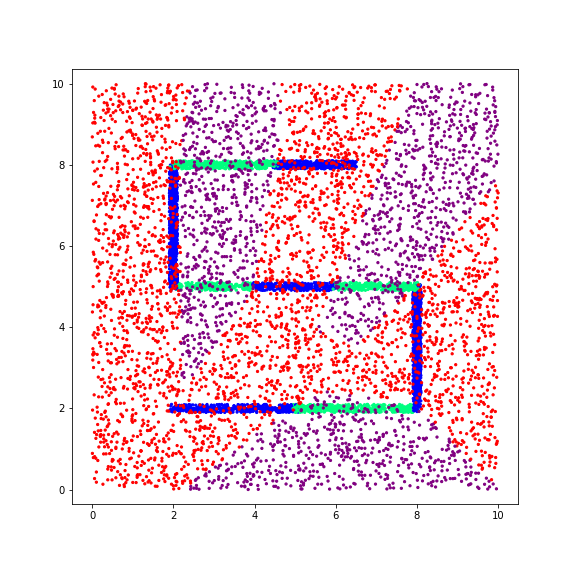}
 \hspace{-.54cm}
 \includegraphics[width=.18\textwidth]{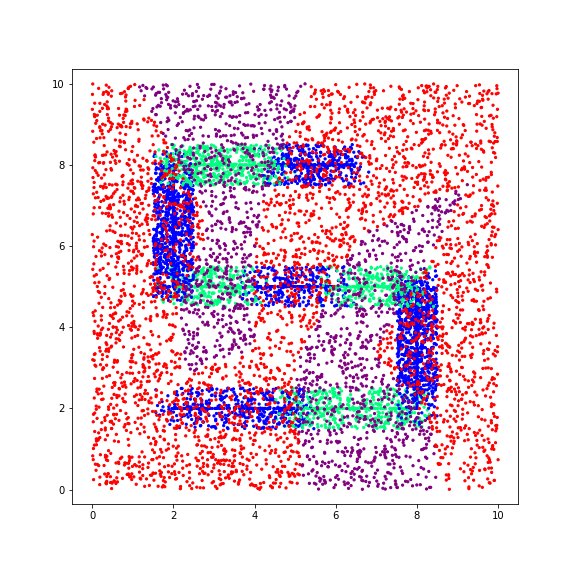}
 \hspace{-.54cm}
 \includegraphics[width=.18\textwidth]{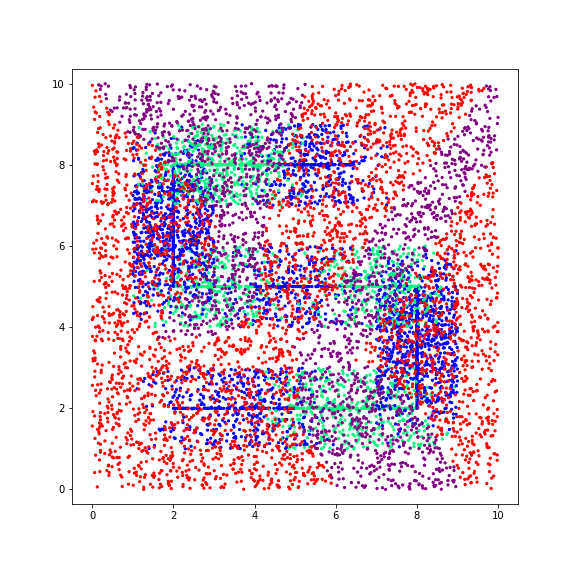}
 \hspace{-.54cm}
 \includegraphics[width=.18\textwidth]{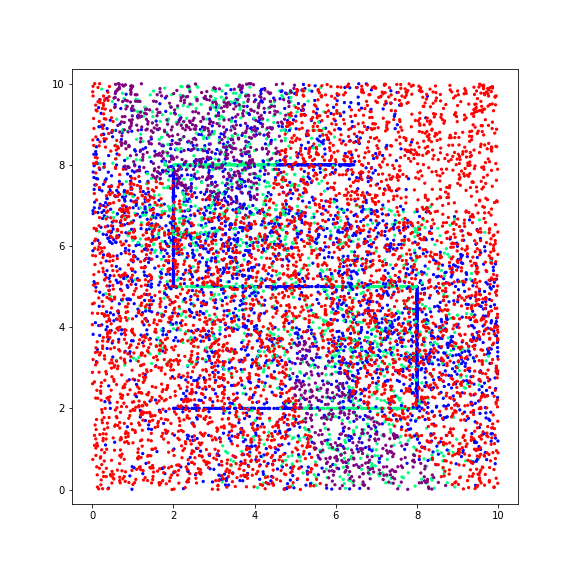}
 \hspace{-.54cm}
 \includegraphics[width=.18\textwidth]{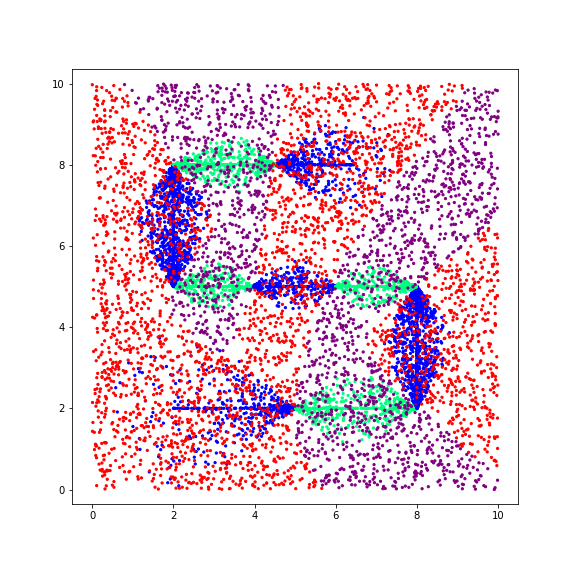}
 
 \includegraphics[width=.18\textwidth]{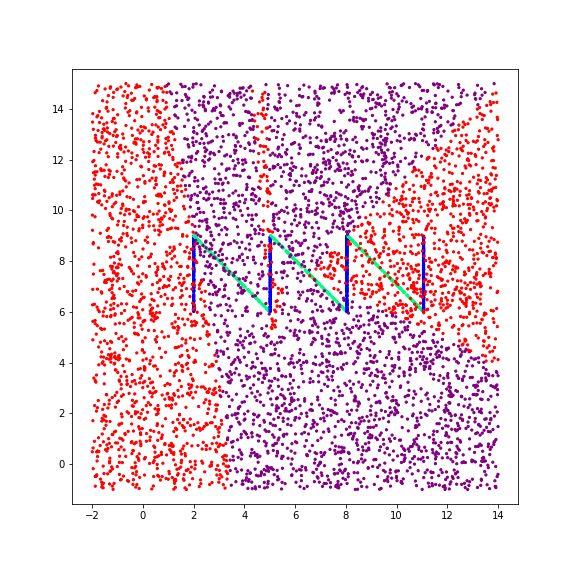}
 \hspace{-.54cm}
 \includegraphics[width=.18\textwidth]{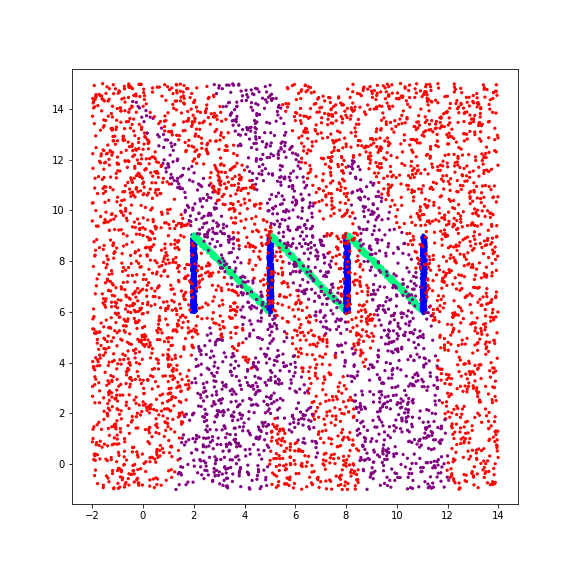}
 \hspace{-.54cm}
 \includegraphics[width=.18\textwidth]{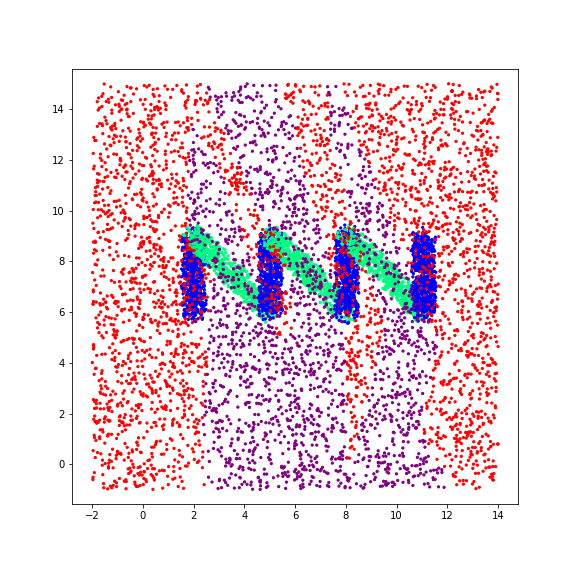}
 \hspace{-.54cm}
 \includegraphics[width=.18\textwidth]{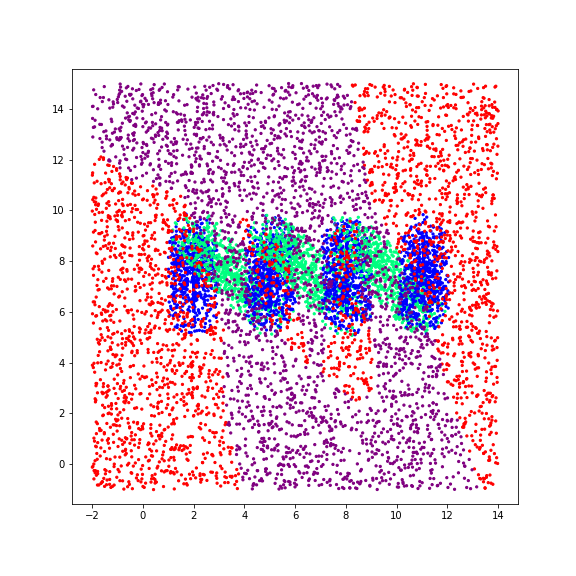}
 \hspace{-.54cm}
 \includegraphics[width=.18\textwidth]{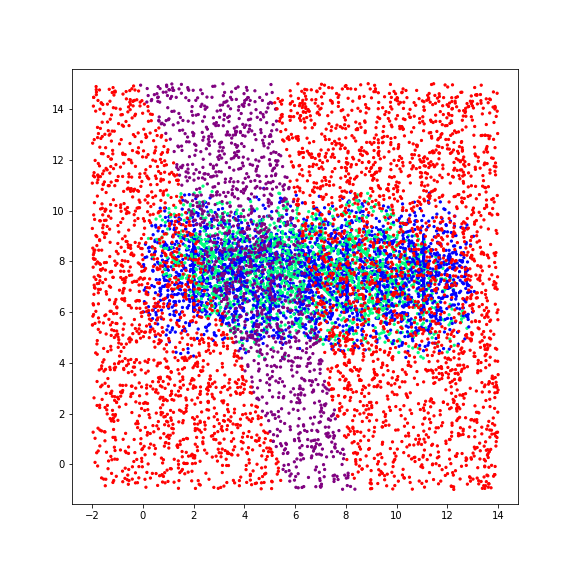}
 \hspace{-.54cm}
 \includegraphics[width=.18\textwidth]{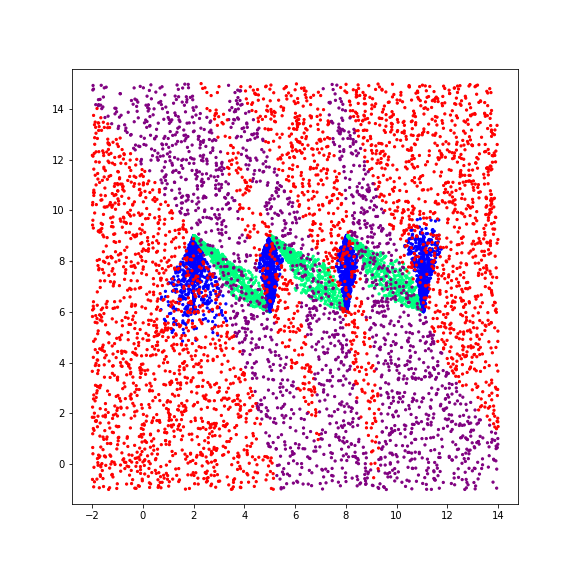}
 
 \includegraphics[width=.155\textwidth]{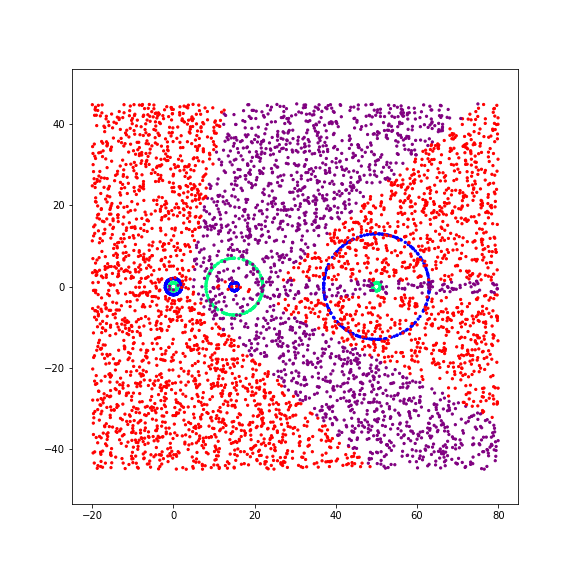}
 \hspace{-.50cm}
 \includegraphics[width=.155\textwidth]{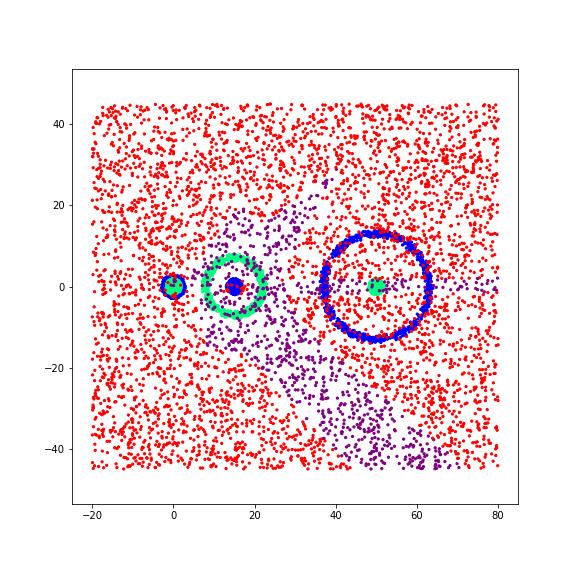}
 \hspace{-.50cm}
 \includegraphics[width=.155\textwidth]{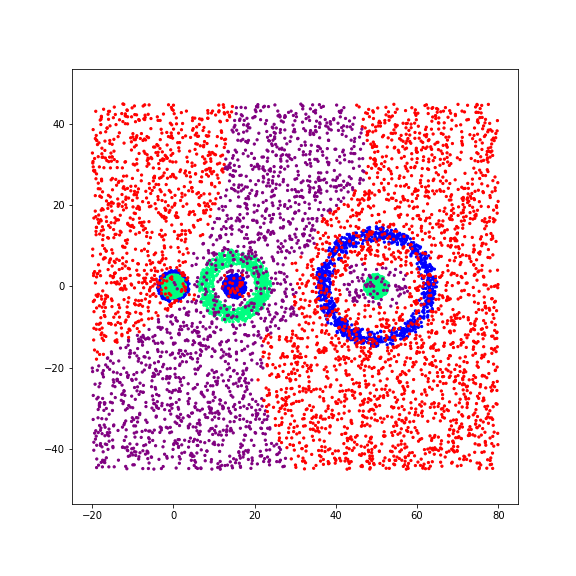}
 \hspace{-.50cm}
 \includegraphics[width=.155\textwidth]{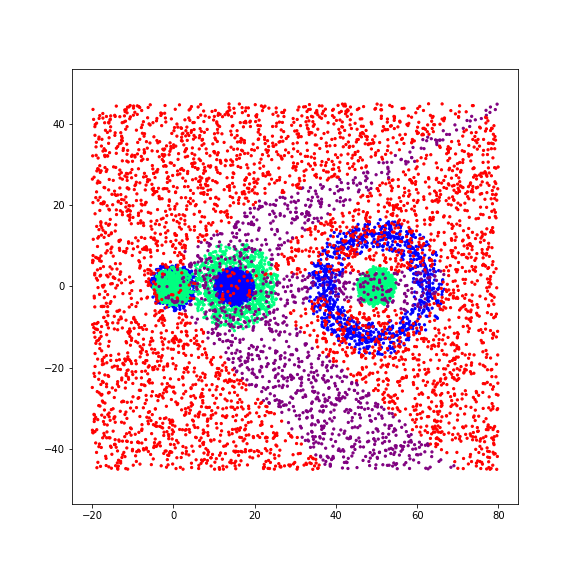}
 \hspace{-.50cm}
 \includegraphics[width=.155\textwidth]{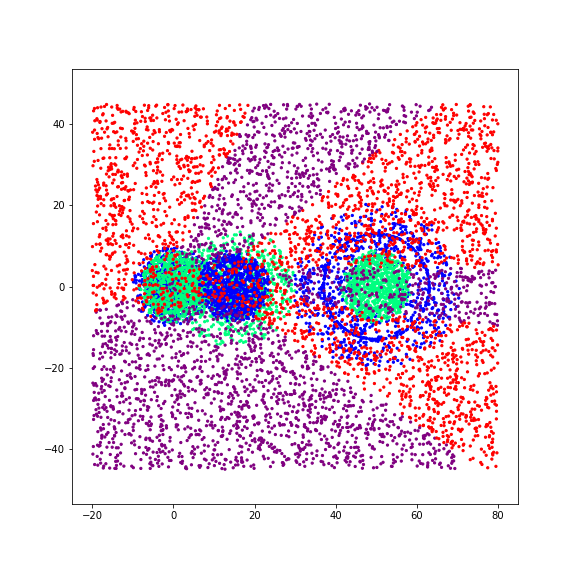}
 \hspace{-.50cm}
 \includegraphics[width=.155\textwidth]{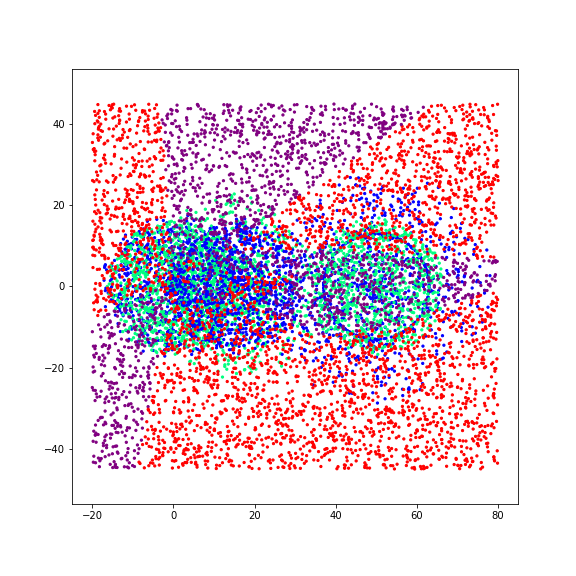}
 \hspace{-.50cm}
 \includegraphics[width=.155\textwidth]{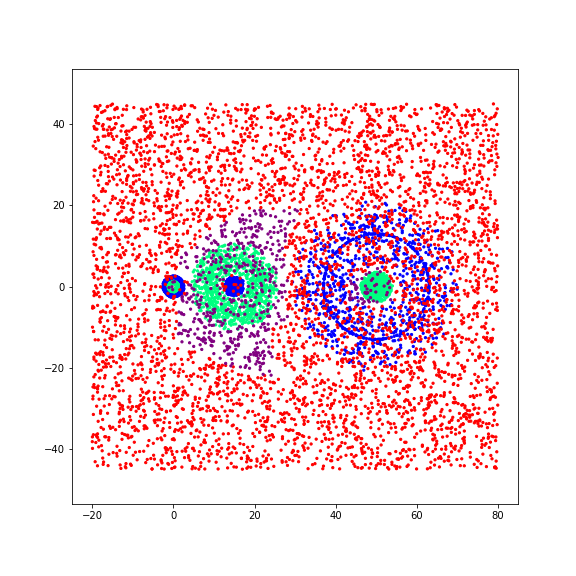}
 
 \includegraphics[width=.155\textwidth]{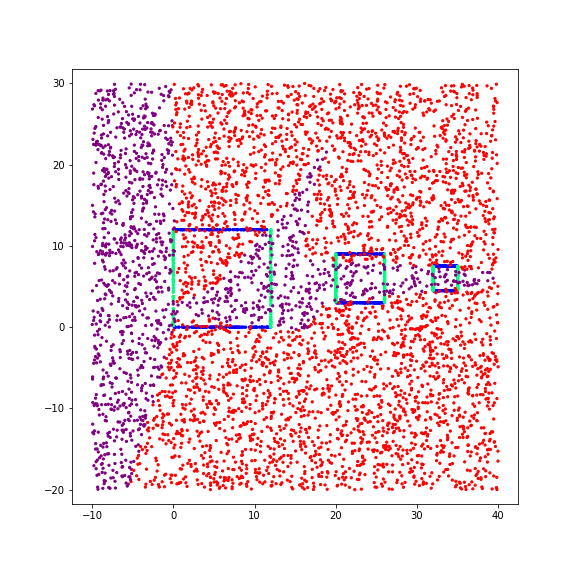}
 \hspace{-.50cm}
 \includegraphics[width=.155\textwidth]{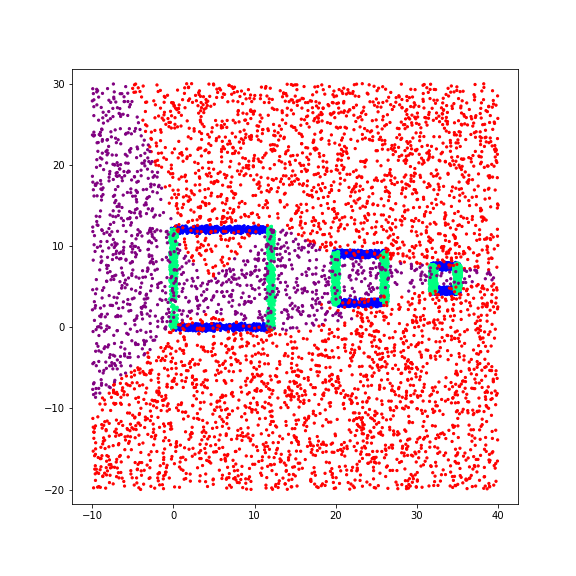}
 \hspace{-.50cm}
 \includegraphics[width=.155\textwidth]{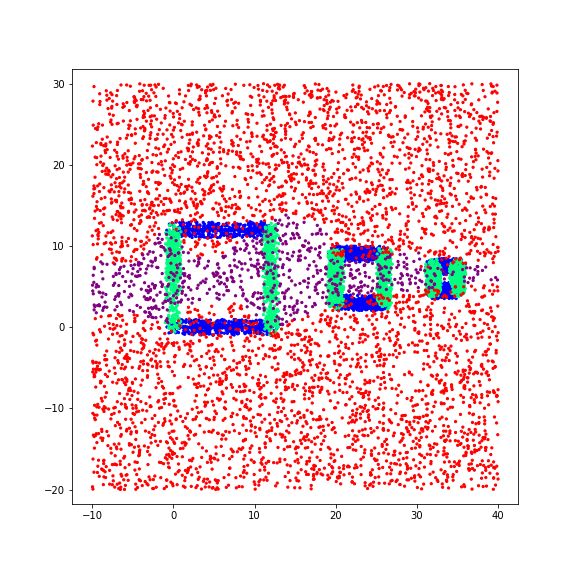}
 \hspace{-.50cm}
 \includegraphics[width=.155\textwidth]{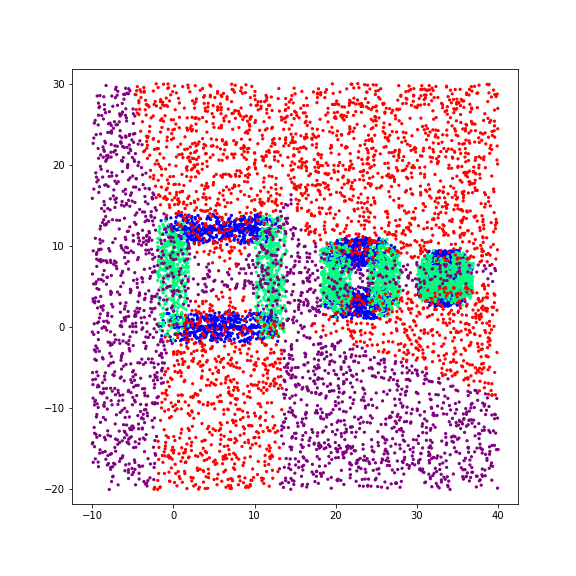}
 \hspace{-.50cm}
 \includegraphics[width=.155\textwidth]{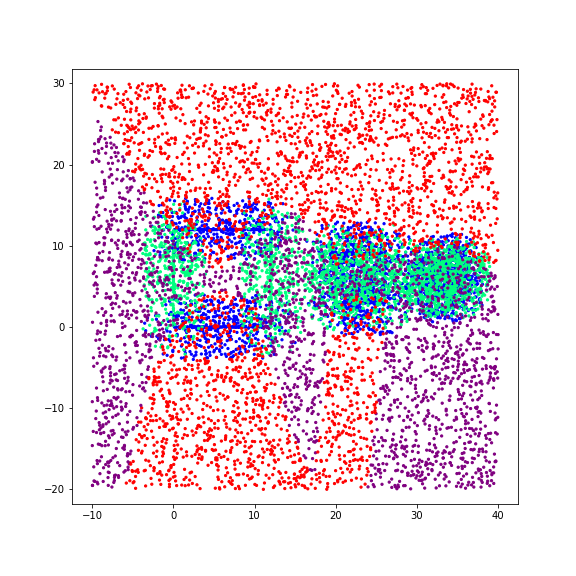}
 \hspace{-.50cm}
 \includegraphics[width=.155\textwidth]{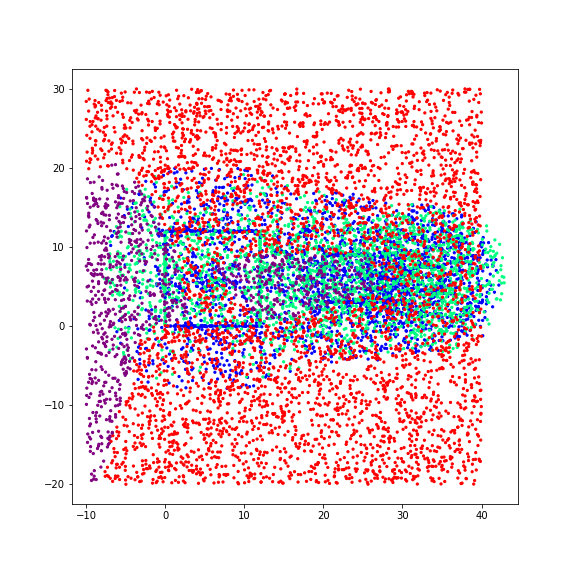}
 \hspace{-.50cm}
 \includegraphics[width=.155\textwidth]{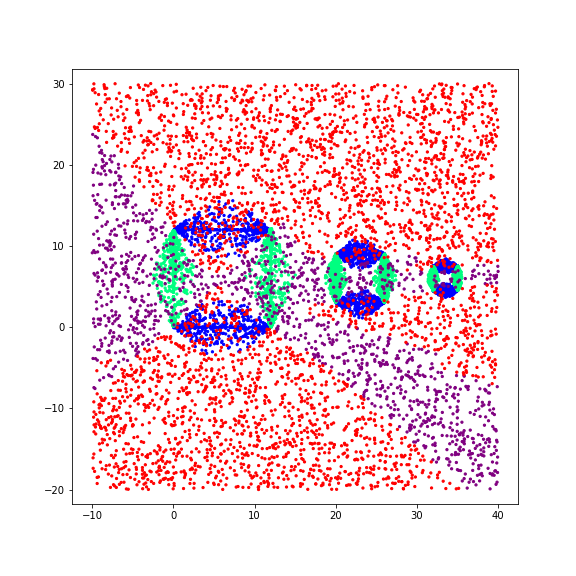}

\end{center}
\caption{ReLU networks trained on data from a one-dimensional manifold in two-dimensional space, labeled using two classes (blue and green here). The various shapes by row: {\bf Sines, S-figure, NNN, circles, boxes}. 
Left-most: original training data; various middle images: training data augmented using increasing expansion parameters; right-most: training data robust-adaptive expanded. 
We use data generated uniformly at random from the ambient space to illustrate the network's labeling (red and purple). Using just original training data, or only slightly augmented data, we observe that the network's decision boundary is often close to the manifold.}
\label{fig:allvisualizations}
\end{figure*}

Analogously we augment the data with an adaptive expansion parameter. The key difference is in the calculation of the radius of the sphere.
A fraction of the distance between the current sample and a nearest neighbor of a different class is used as the radius for the sphere generation. 
Each of the middle columns in Figure \ref{fig:allvisualizations} corresponds to augmentation with a fixed expansion parameter, while the last column shows the $2/3$-adaptive robust augmentation of the training data. The original training dataset contains $1000$ training points and the augmented datasets $5000$ data points each.

For the various networks, we evaluate, binary loss, robust loss with a fixed robustness parameter and the adaptive robust loss.
We also evaluate the adaptive robust loss on the various trained networks. 
To estimate the adaptive robust loss at a point $(x_1, x_2)$, we determine its distance $\rho$ to a point in the dataset with a different label and then generate 10 test points uniformly at random from a ball of radius $0.5\rho$. If one of these gets a different label than $(x_1, x_2)$ by the network (or if the point is mislabeled itself) it suffers adaptive robust loss $1$. The table in Figure \ref{fig:synthetictable} summarizes the binary and adaptive robust losses of the various networks. We see that the adaptive augmentation leads consistently to the lowest binary (always rank 1) and low adaptive robust loss (rank 1 and once rank 2).
This shows that the adaptive augmentation not only is not in conflict with accuracy, but empirically improves accuracy of a trained network.

Finally, we also trained ReLU neural networks on several real-world data sets from the UCI repository. 
For each dataset, we normalized the features to take values in [0,1].
As in the experiments on the synthetic data, we trained the networks on the original data, as well as various augmented datasets, including using the $2/3$-adaptive augmentation. The datasets were split into training and test data with a ratio of $80-20$ respectively.
In Figures \ref{fig:synthetictable} and \ref{fig:UCItable},
we report the binary and robust losses of these networks.
We observe, again, that the robust augmentation promotes the best performance in terms of $0/1$ accuracy.
Additionally, the adaptive robust loss is close to the best adaptive robust loss achieved with a fixed expansion parameter on each dataset. Using the adaptive augmentation can thus serve to save needing to search for an optimal expansion parameter on different tasks.

In summary, our initial experimental explorations here showed that the adaptive augmentation consistently yielded
a robust predictor with best $0/1$-loss. This confirms
the intended design of an adaptive robustness and data augmentation paradigm that avoids the undesirable tradeoffs between robustness and accuracy.

\begin{figure*}[b]
 \includegraphics[trim = 0  100 0 0, clip, width=\textwidth]{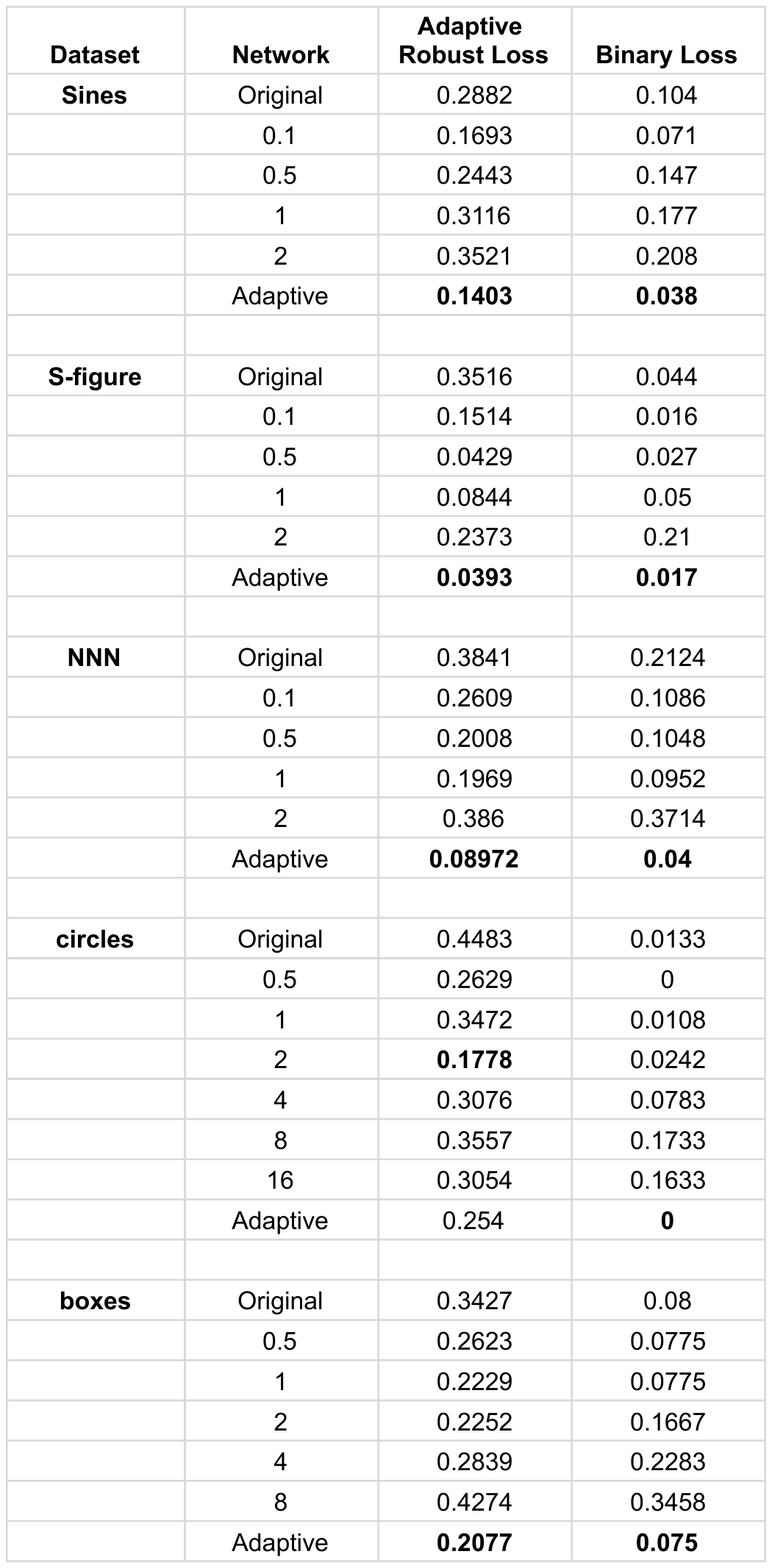}
 \caption{Overview on the binary and adaptive robust losses of the networks trained on trained on the various synthetic datasets with various augmentations.}
 \label{fig:synthetictable}
\end{figure*}

\begin{figure*}[b]
 \includegraphics[trim = 0  100 0 0, clip, width=\textwidth]{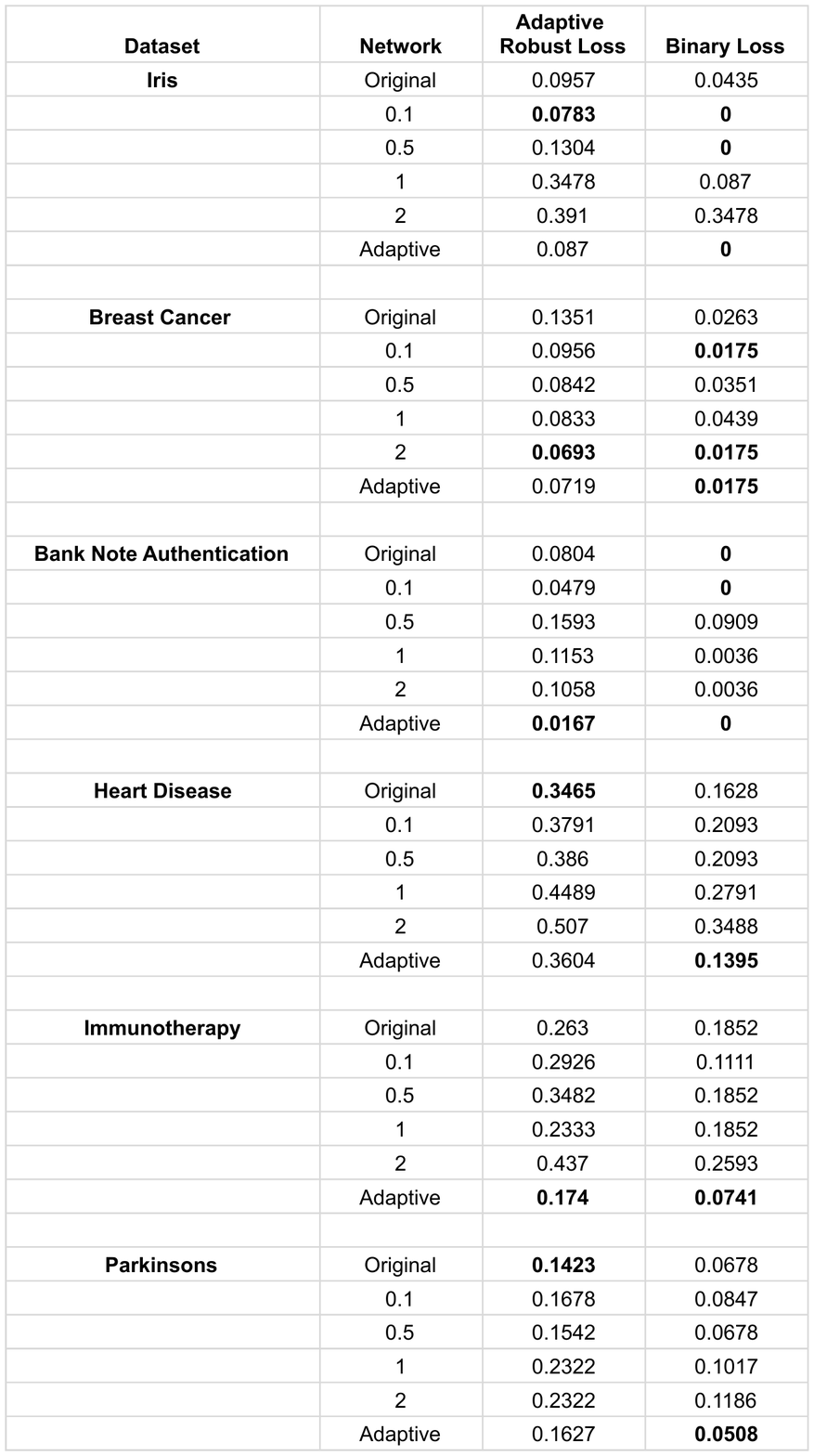}
 \caption{Overview on the binary and adaptive robust losses of the networks trained on trained on the various UCI datasets with various augmentations.}
 \label{fig:UCItable}
\end{figure*}

\section{Concluding Remarks}

In this work, we initiate studying adversarial robustness as an \emph{adaptive requirement}. Through a series of constructions where optimal classifiers for robust loss and $0/1$-loss differ drastically, we motivate re-framing adversarial robustness as a requirement that should be in line with the underlying distribution's margin properties. We propose a formal notion of such an adaptive loss, as well as an accompanying empirical version and implied data-augmentation paradigm.
As a first sound justification of this proposal, we prove that this type of adaptive data-augmentation maintains consistency of a non-parametric method (namely $1$-nearest neighbor classification under deterministic labels). 
We believe this to be a natural and useful take on dealing with the inconsistencies (eg in terms of growing loss-class capacities, computational impossibilities, or diverging Bayes predictors) that earlier theoretical studies on learning under adversarial loss have exhibited. 
\end{document}